\documentclass{article}
\usepackage{fieldnotes}
\usepackage[utf8]{inputenc}
\usepackage[T1]{fontenc}
\usepackage{amsmath,amssymb,amsthm,mathtools}
\usepackage{microtype,booktabs,graphicx,enumitem}
\usepackage[numbers,sort&compress]{natbib}
\usepackage[hidelinks]{hyperref}
\usepackage{url}
\setlist{nosep}
\newtheorem{theorem}{Theorem}[section]

\newtheorem{proposition}[theorem]{Proposition}
\newtheorem{corollary}[theorem]{Corollary}
\theoremstyle{remark}
\newtheorem{remark}[theorem]{Remark}
\newcommand{\R}{\mathbb R}
\newcommand{\Sph}{\mathbb S^2}
\newcommand{\SO}{\mathrm{SO}(3)}
\newcommand{\df}{d_{\mathrm{form}}}
\newcommand{\DB}{d_B}
\newcommand{\eps}{\varepsilon}
\DeclareMathOperator{\rank}{rank}

\preprintlabel{Preprint}
\preprintmeta{13 October 2026}
\runningtitle{Quotient Geometry and Temporal Persistence}
\title{Quotient Geometry and Persistence-Stable Metrics for Swarm Configurations}
\author{
  \textbf{Mark M. Bailey} \\
  Associate Director, Center for the Future of AI, Mind, and Society \\
  Florida Atlantic University\\
  Boca Raton, FL, USA\\
  \texttt{baileymark@fau.edu}
}
\begin{document}
\maketitle
\begin{abstract}
Can temporal information recover geometry lost by a persistence diagram?
For $n\ge3$ points on the unit sphere evolving under known homogeneous
complete consensus, we prove generic local recovery modulo rotations and
permutations from two ordinary zero-dimensional Vietoris--Rips persistence
snapshots. The structural result applies to every labeled tree: its edge
inner products and their first temporal derivatives have differential
rank $2n-3$, exactly the dimension of the regular formation quotient.
An explicit pole--equator witness proves this simultaneously for all
tree types after taking a finite intersection of open dense sets.
Although one snapshot has generic local fibers of dimension $n-2$, two
sufficiently close snapshots are injective on a common neighborhood.
The optimal inverse Lipschitz constant grows as $\Theta(1/|h|)$ when the
sampling interval tends to zero. A verified left inverse, interval
derivative bounds, and explicit persistence-matching margins certify
neighborhoods and time intervals at rational reference formations.
The results are local and model-specific; reflection ambiguity and
centered rank obstructions remain.
\end{abstract}
\par

\section{Introduction}\label{sec:background}
Collective motion raises a geometric inference problem: which features
of a group's relative arrangement can be determined from a compressed
observation of the group? Formation control and consensus theory study
how interactions shape collective configurations
\cite{oh2015survey,ren2005consensus}; configuration-space methods describe
reconfiguration as motion through admissible states
\cite{ghristpeterson2007reconfiguration}. For interchangeable agents, the
geometric object of interest is often a formation rather than an ordered
list of coordinates. Relabeling the agents or changing a common reference
frame should then leave the description unchanged. A useful observation
must respect those symmetries while retaining the distinctions relevant
to the inference problem.

Persistent homology offers a natural class of such observations. It
summarizes the connectivity and higher-dimensional structure of a point
cloud across spatial scales, without requiring a single proximity
threshold \cite{edelsbrunner2002topological,zomorodian2005computing}.
Applications to biological aggregation models have shown how these
summaries can reveal changes in collective organization that are not
captured by the order parameters considered in those studies
\cite{topaz2015aggregation}. Persistence diagrams also satisfy forward
stability estimates: small changes in the underlying geometry produce
small changes in the diagrams
\cite{cohensteiner2007stability,chazal2014persistence}. These properties
motivate persistence as a representation of evolving formations.

The use of a stable summary nevertheless leaves an inverse question.
If two formations have similar diagrams, must their geometries be
similar after accounting for the allowed symmetries? Forward stability
does not supply this implication. A descriptor can vary continuously
while assigning the same value to an entire family of geometrically
different states. Distinguishing this loss of information from sensitivity
to perturbations is necessary if persistence observations are to support
geometric inference. This paper asks whether a known evolution law can
make an otherwise insufficient observation locally determine the
formation, and how the resulting inverse depends on the separation
between observation times.

\subsection{The information deficit in a static observation}
We study $n\ge3$ points on the unit sphere, identifying configurations
related by a common rotation and a permutation. This quotient is natural
for directional formations when absolute orientation and individual
labels are outside the inference target. Its regular part has dimension
$2n-3$: each point supplies two degrees of freedom, and common rotations
remove three. The discrete permutation symmetry changes the global
identifications but not this local dimension.

The observation is the ordinary zero-dimensional Vietoris--Rips
persistence diagram computed from spherical geodesic distances.
Zero-dimensional persistence is a deliberately economical choice: it
records how connected components merge as the proximity radius grows.
For $n$ distinct points it has one essential class and $n-1$ finite death
times, equal to half the edge lengths of a minimum spanning tree (MST)
under our filtration convention. Most pairwise distances, and the
incidence labels of the retained edges, are absent from the diagram.
Even for a spherical triangle, the two MST lengths leave its third side
undetermined within a regular local branch.

Consequently, a single diagram cannot be locally injective on an open
$(2n-3)$-dimensional formation neighborhood. More precisely, the generic
static observation has rank $n-1$ and fibers of dimension $n-2$, as
explained in Section~\ref{sec:structural}. These fibers consist of
formations that can deform while preserving the observed tree lengths.
The deficit is intrinsic to this observation, rather than a consequence
of numerical precision or an imperfect reconstruction method. It also
makes $H_0$ a useful setting in which to isolate the information supplied
by time: any recovery result must explain how the additional observation
resolves directions that the first diagram cannot distinguish.

A second snapshot can provide such constraints if the prescribed
dynamics make formations within a static fiber evolve differently in
observation space. The relevant information is then the response of the
same geometric summary to the known vector field. This is the
observability viewpoint: properties of the initial state can become
visible through the temporal derivatives of an incomplete observation
\cite{hermann}. Whether that mechanism succeeds depends on the particular
observable and evolution law. It does not follow merely from collecting
more copies of a stable descriptor.

\subsection{A fixed model and the main results}
Our model is homogeneous complete spherical consensus,
\begin{equation}\label{eq:flow}
 \dot p_i=(I-p_ip_i^{\mathsf T})S=S-s_ip_i,
 \qquad S=\sum_{j=0}^{n-1}p_j,\qquad s_i=p_i\cdot S.
\end{equation}
This is the complete-graph, constant-coupling specialization of an
established class of spherical consensus protocols \cite{markdahl}.
The coupling coefficient is fixed at one throughout. The spherical
state space admits a directional interpretation, and the flow preserves
unit lengths and the rotation and permutation symmetries of the
formation model. Here its role is to provide an explicit, fully known
evolution on which a local inverse analysis can be carried out.

The mechanism is visible in the identity
\[
 \frac{d}{dt}(p_i\cdot p_j)
   =(1-p_i\cdot p_j)(s_i+s_j),
 \qquad s_i=\sum_k p_i\cdot p_k.
\]
Although an MST retains only selected pairwise distances, the rate of
change of each retained edge depends on the wider formation through
$S$. Thus temporal measurements can expose geometric information that
is absent from the static tree. The nontrivial issue is whether these
particular rates remove \emph{all} generic infinitesimal ambiguities,
rather than merely reducing their number.

At each observation time the data are only the ordinary, unlabeled
$H_0$ diagram of the entire point cloud. We add no agent correspondences,
edge incidence labels, or higher-dimensional persistence to those data.
Writing $\varphi_h$ for the flow over a known time $h$, the inverse
problem is to recover the initial formation locally from
\[
 \bigl(D_0(p),D_0(\varphi_h(p))\bigr).
\]
Within a neighborhood where the MST and the order of its edge lengths
remain fixed, the sorted finite deaths give smooth observation
coordinates. This local representation permits differential analysis
without treating the unobserved vertex labels as additional measurements.

The mathematical argument has three linked parts.
\begin{enumerate}
\item \textbf{Structural identifiability for every tree.}
Theorem~\ref{thm:rank} proves that for each $n\ge3$ there is a common
open dense set of configurations on which, for every labeled spanning
tree $T$, the map
\[
 p\longmapsto
 \bigl((p_i\cdot p_j)_{\{i,j\}\in E(T)},
       (s_i+s_j)_{\{i,j\}\in E(T)}\bigr)
\]
has differential rank $2n-3$, with only infinitesimal rotations
in its kernel. A pole--equator witness and an algebraic genericity
argument establish the result for all tree types. The quantification
over every tree matters because the observed MST depends on the
configuration; the proof does not require selecting a favorable tree.
The identity above transfers this rank to the finite deaths and their
first temporal derivatives. Genericity concerns the initial formation
under one fixed observable and one fixed vector field.

\item \textbf{Two-snapshot recovery and sharp conditioning.}
A uniform perturbation argument converts the first-derivative result
into local recovery from two snapshots on a common neighborhood for all
sufficiently small nonzero $h$. Theorem~\ref{thm:rate} further proves
that the smallest inverse Lipschitz constant, measured using formation
distance and the larger of the two bottleneck discrepancies, grows as
$\Theta(1/|h|)$ as $h\to0$. Along a direction invisible to the first
snapshot, the additional first-order information vanishes linearly as
the observations coalesce. The matching upper and lower bounds therefore
separate identifiability from conditioning: exact recovery can remain
locally unique while becoming increasingly sensitive to observation
error. This is a statement about the local inverse on model-generated
data, not a statistical performance guarantee for a noisy estimator.

\item \textbf{Validated neighborhoods and delays.}
Theorem~\ref{loc:thm:local} gives checkable sufficient conditions for a
full formation neighborhood and an interval of admissible sampling
delays. A verified left inverse and interval derivative bounds control
the observation map, while trajectory and matching margins preserve the
MST branch and exclude matches to the persistence diagonal. These
conditions go beyond a rank calculation at one reference point. Exact
rational examples and outward-rounded interval calculations instantiate
the criterion and compare it with an analytic baseline.
\end{enumerate}
Together these results show how a known evolution law can make two compressed
observations locally informative enough to recover geometry, quantify
the cost of closely spaced sampling, and provide explicit domains on
which that conclusion is certified.

\subsection{Relation to prior work}
Static inverse persistence and its relation to rigidity are established
subjects. Beers, Harrington, Leygonie, Lim, and Theran \cite{beers} study
fibers of persistence for Euclidean point clouds. Their Theorem~58
characterizes generic local identifiability up to Euclidean isometry
from the full Vietoris--Rips persistence data, over homological degrees,
by rigidity of a critical framework. Gameiro, Hiraoka, and Obayashi \cite{gameiro} develop
continuation methods for point clouds with prescribed nearby diagrams.
Our observation retains only $H_0$, and the structural step completes
the static constraints of every tree using a Lie derivative under
\eqref{eq:flow}.

Other inverse results obtain geometric information from richer
collections of persistence observations. The persistent homology
transform uses directional height filtrations and is injective on the
studied classes of piecewise linear shapes \cite{turner2014pht}.
Distributed persistence gives quantitative inverse bounds from
subset-indexed collections of diagrams under covering and closure
hypotheses, with subsets compared through a common map
\cite{solomon2023distributed}. Those
observations retain directional or subset information. The question
here is whether a known evolution supplies sufficient additional
constraints when each observation remains one ordinary metric
persistence diagram of the entire formation.

Dynamic persistence has a substantial forward theory. Kim and M\'emoli
\cite{kim} construct spatiotemporal persistence for dynamic metric spaces
and prove stability. More recent constructions include dynamic
curvature-set persistence \cite{belova} and extended-zigzag persistence
landscapes \cite{flammer}. These approaches motivate the analysis of
geometric data over time. The present problem concerns local recovery
of an initial formation from two ordinary diagrams under a specified
flow; it does not introduce a new spatiotemporal persistence invariant.

Observation derivatives are classical in nonlinear observability
\cite{hermann}. Takens' embedding results include delay and derivative
coordinates \cite{takens}, while geometry-preserving delay embeddings
and sampling tradeoffs have also been studied \cite{eftekhari}.
These frameworks explain why dynamics can enrich observations, but do
not establish the required rank for the fixed observable and fixed
vector field considered here. Our contribution is the all-tree rank
calculation, its two-snapshot and sharp-conditioning consequences, and
the explicit conditions that validate them on a neighborhood.
Likewise, preconditioning and interval derivative enclosures are
standard tools of verified numerics, with roots in interval Newton
methods \cite{krawczyk}; the certificate specializes them to the flow,
quotient gauge, and persistence-matching conditions.

\subsection{Scope and organization}
The results concern local identifiability under known dynamics, with a
fixed number of points and known observation times. They do not assert
global reconstruction: reflection-related formations have identical
persistence time series, and distinct local branches can remain
indistinguishable. Centered configurations also exhibit a nongeneric
first-order rank obstruction; that rank loss alone does not establish
nonlinear noninjectivity. The inverse estimates compare admissible model
trajectories within a common neighborhood. They neither identify that
neighborhood from an arbitrary noisy diagram nor construct an estimator
for such data. The computations supply exact and interval certificates,
rather than trajectory simulations or empirical recovery benchmarks.

The formation metric, quotient construction, MST description of $H_0$,
and forward stability serve as established background.
Section~\ref{sec:setting} specifies this geometry and the observation
convention. Sections~\ref{sec:structural}--\ref{sec:certificate} develop
the structural theorem, sharp sampling-rate conditioning, and validated
local criterion. Section~\ref{sec:computation} gives the computational
methodology and exact results, and Section~\ref{sec:scope} discusses
obstructions and limitations. Appendices~\ref{app:baseline}
and~\ref{app:software} provide the analytic baseline and reproduction
protocol.

\section{Formation geometry and persistence}\label{sec:setting}
\subsection{The formation quotient and its metric}
Let $(M,d_M)$ be a metric space and let a compact group $G$ act
continuously by isometries on $M$. For a fixed number $n$ of agents,
the formation space and its matching metric are
\begin{align}
 \mathcal S_n(M,G)&=M^n/(G\times\mathfrak S_n),\label{eq:general-shape}\\
 d_{M,G}([x],[y])&=\min_{g\in G,\,\pi\in\mathfrak S_n}
       \max_i d_M(gx_i,y_{\pi(i)}).\label{eq:general-metric}
\end{align}
The minimum is attained by compactness. This is the standard orbit
metric induced by the maximum product metric: positivity follows from
attainment, and composing alignments proves the triangle inequality.
Its topology is the quotient topology. The quotient is complete if $M$
is complete, compact if $M$ is compact, and geodesic if $M$ is geodesic.
For the last assertion, choose a minimizing alignment and project a
product geodesic between the aligned representatives. In the smooth
Riemannian setting with a compact Lie group, the orbit-type geometry is
standard \cite{alekseevsky2003orbitspaces}. We use these facts as a setting
for the inverse problem.

The space $M^n$ allows collisions. Persistence below is evaluated on the
collision-free subset
\[
 \operatorname{Conf}_n(M)=\{x\in M^n:x_i\ne x_j\text{ for }i\ne j\},
\]
where $X_x=\{x_0,\ldots,x_{n-1}\}$, with the induced distance, is a
genuine $n$-point metric space. This distinction avoids treating a
zero-distance pair of distinct indices as a metric space. The group
$G$ is part of the formation model: $d_{M,G}$ measures ambient alignment
and need not equal an intrinsic distance between finite metric spaces.

\paragraph{Sphere and torus models.}
For directional formations, take $M=\Sph$ with the geodesic metric
$d_{\Sph}(p,q)=\arccos(p\cdot q)$ and $G=\SO$. This removes a common
orientation while retaining relative spherical geometry. A flat phase
model takes $M=\mathbb T^k=\R^k/(2\pi\mathbb Z)^k$, with
\[
 d_{\mathbb T^k}(x,y)=\min_{z\in(2\pi\mathbb Z)^k}\|x-y+z\|_2,
\]
and $G=\mathbb T^k$ acting by translation. Both fit
\eqref{eq:general-shape}--\eqref{eq:general-metric}. The temporal recovery
theorems proved here apply only to the spherical consensus
model~\eqref{eq:flow}; the spherical representation alone does not model orbital dynamics.

\subsection{Forward stability and the information in one diagram}
Use Rips radius $r$ with an edge present when its length is at most $2r$,
and use homology over a fixed field. Let $D_k(X)$ denote the resulting
persistence diagram and set $\Phi_k([x])=D_k(X_x)$; ambient isometries and
relabeling preserve it. The standard forward estimate is
\begin{equation}\label{eq:forward}
 d_B(\Phi_k([x]),\Phi_k([y]))
       \le d_{\rm GH}(X_x,X_y)\le d_{M,G}([x],[y]).
\end{equation}
Indeed, an aligned bijection with vertex error at most $\varepsilon$
has pairwise-distance distortion at most $2\varepsilon$, giving the
second inequality. The first is Rips stability after rescaling the
usual edge-length parameter \cite{chazal2014persistence}. In particular,
a Lipschitz formation trajectory has a persistence trajectory with no
larger Lipschitz constant. Neither statement gives inverse control.

Zero-dimensional persistence has one essential class and $m=n-1$
finite points $(0,d_e)$, where the $d_e$ are half the edge lengths of an
MST. This follows from Kruskal's algorithm and component merging; see
also \cite{beers}. It explains why a stable diagram can still compress
the formation substantially.

Ordinary bottleneck distance allows matching to the diagonal. A finite
point $(0,d)$ costs $d/2$ to delete. For example, two spherical
two-point formations with separations $\delta_x,\delta_y\in(0,\pi]$
satisfy
\begin{align}
 d_{\Sph,\SO}([x],[y])&=\tfrac12|\delta_x-\delta_y|,\label{eq:two-point-form}\\
 d_B(D_0(X_x),D_0(X_y))
  &=\min\{\tfrac12|\delta_x-\delta_y|,
                    \tfrac14\max(\delta_x,\delta_y)\}.\label{eq:two-point-correction}
\end{align}
For \eqref{eq:two-point-form}, the triangle inequality gives the lower
bound, and aligning the arc midpoints and great-circle directions
attains it. Equation~\eqref{eq:two-point-correction} compares direct
matching with deleting both finite points. Thus even this simple
injective signature is not globally isometric. We identify bottleneck
distance with differences of sorted deaths only after imposing a local
condition that excludes diagonal matching.

\subsection{The spherical dynamics and regular local coordinates}
From now on fix $n\ge3$, put $\mathcal C_n=(\Sph)^n$, index vertices by
$0,\ldots,n-1$, and write $q_{ij}=p_i\cdot p_j$. Let $f$ be the vector
field in~\eqref{eq:flow} and $L_fg=Dg\,f$ the Lie derivative of an
observation $g$. The flow $\varphi_t$ is tangent to $\mathcal C_n$,
complete, and equivariant under $O(3)$ and vertex permutations. Write
$D_0(p)=D_0(X_p)$ and abbreviate the spherical formation metric by
\begin{equation}\label{eq:formation}
 \df([p],[r])=d_{\Sph,\SO}([p],[r])
   =\min_{Q\in\SO,\,\pi\in\mathfrak S_n}
       \max_i d_{\Sph}(Qp_i,r_{\pi(i)}).
\end{equation}

All recovery statements exclude coincident and antipodal pairs. Write
$H(p)\in\R^m$ for the finite deaths in increasing order. A strict
ordering of all pairwise distances fixes the MST and its ordered
branches locally. Accordingly, $H$ is smooth inside such a chamber
although it is not globally smooth.

At a configuration with trivial rotation--permutation stabilizer, the
quotient has local dimension $d=2n-3$. Slice coordinates and the
formation metric are locally bi-Lipschitz equivalent; see the compact
isometric-action framework in \cite{alekseevsky2003orbitspaces}.
Distinct pairwise distances imply trivial permutation stabilizers for
$n\ge3$, and two nonparallel vertices make the rotation action free.
Our generic set is open and dense in $\mathcal C_n$; it excludes the
proper algebraic rank loci proved below and the preceding degeneracies.
Genericity is in the reference configuration, rather than in the choice
of observation function or vector field.

\section{The structural theorem for arbitrary trees}\label{sec:structural}
For a labeled tree $T$ on $\{0,\ldots,n-1\}$ define the polynomial map
\begin{equation}\label{eq:FT}
 F_T(p)=\big(q_{ij},s_i+s_j\big)_{\{i,j\}\in E(T)}
       \in\R^{2n-2}.
\end{equation}
Here the ordering of the edges is arbitrary and fixed. The rotations generate tangent vectors
\[
 \mathcal R_p=\{(\omega\times p_i)_{i=0}^{n-1}:\omega\in\R^3\}.
\]
For configurations containing two nonparallel points this space has dimension three.

\begin{theorem}[Generic rank for all trees]\label{thm:rank}
For every $n\ge3$ there is an open dense subset $\mathcal U_n\subset\mathcal C_n$ such that, for every $p\in\mathcal U_n$ and every labeled spanning tree $T$,
\[
 \ker\big(DF_T(p)|_{T_p\mathcal C_n}\big)=\mathcal R_p,
 \qquad \rank\big(DF_T(p)|_{T_p\mathcal C_n}\big)=2n-3.
\]
In particular, the result applies to the unique MST at every generic formation.
\end{theorem}

\begin{proof}
Fix $T$ and choose a leaf, denoted $0$ for this proof, with neighbor $j$. Put the leaf at the north pole and all other vertices on the equator:
\[
 p_0=e_3,\qquad p_i=(\cos\theta_i,\sin\theta_i,0)\quad(i\ne0).
\]
Choose angles such that $\sin(\theta_i-\theta_k)\ne0$ on every edge of $T-\{0\}$. This is possible, for example with distinct angles in an interval of length less than $\pi$. The witness need not have $T$ as its MST.

Let $v\in T_p\mathcal C_n$ satisfy $DF_T(p)v=0$. Rotational invariance permits subtracting an infinitesimal rotation so that $v_0=0$. Write
\[
 v_i=\alpha_i t_i+\beta_i e_3,\qquad
 t_i=(-\sin\theta_i,\cos\theta_i,0)\quad(i\ne0).
\]
For an edge $\{i,k\}$ of $T-\{0\}$, the vanishing derivative of its dot product gives
\[
 0=\delta q_{ik}=\sin(\theta_k-\theta_i)(\alpha_i-\alpha_k).
\]
Deleting a leaf leaves a connected tree, so all $\alpha_i$ coincide. Subtract the residual rotation about $e_3$ to obtain $v_i=\beta_i e_3$ at every equatorial vertex. The leaf-edge constraint is
\begin{equation}\label{eq:anchor}
 0=\delta q_{0j}=\beta_j.
\end{equation}
At the witness $S\cdot e_3=1$ and $\delta S=(\sum_{i\ne0}\beta_i)e_3$. Thus
\begin{equation}\label{eq:dswitness}
 \delta s_i=v_i\cdot S+p_i\cdot\delta S=\beta_i\quad(i\ne0).
\end{equation}
The remaining observed coordinates imply $\delta s_i+\delta s_k=0$ on every tree edge. By \eqref{eq:anchor}--\eqref{eq:dswitness}, $\delta s_j=0$, and connectedness then forces all $\delta s_i=0$. Therefore all $\beta_i=0$. The entire kernel at the witness consists of the three rotations, proving rank $2n-3$ there.

For completeness, the genericity step can be carried out without choosing a global tangent frame. Extend $F_T$ polynomially to $(\R^3)^n$ and append the $n$ constraint functions $c_i(p)=\tfrac12(\|p_i\|^2-1)$. The matrix $D(c,F_T)$ has polynomial entries. At the witness its rank is $n+(2n-3)=3n-3$, so one of its $(3n-3)$-minors is nonzero. Its restriction to the connected real-analytic manifold $\mathcal C_n$ is therefore not identically zero. The nonvanishing set is open and dense. On its intersection with the open dense set of configurations containing two nonparallel points, rotation invariance gives the matching upper bound, hence the asserted rank. Finally there are finitely many labeled trees. Intersect their open dense full-rank sets to obtain $\mathcal U_n$.
\end{proof}

\begin{remark}[An explicit witness minor]
The witness proof can also be read as a determinant certificate. In the gauge $v_0=0$, $\alpha_j=0$, use all $n-1$ dot-product rows and the $n-2$ rows $\delta(s_i+s_k)$ on $T-\{0\}$. Order the variables as the $n-2$ remaining $\alpha_i$ and the $n-1$ variables $\beta_i$. The matrix is block triangular. Its first block is a weighted reduced incidence matrix of $T-\{0\}$; its determinant is, up to sign, the product of $\sin(\theta_k-\theta_i)$ over those edges. Its second block consists of $\beta_j$ and the sums $\beta_i+\beta_k$ on $T-\{0\}$. Rooting the tree at $j$ and eliminating leaves shows that this block has determinant $\pm1$. Hence the selected determinant equals
\[
 \pm\prod_{\{i,k\}\in E(T-\{0\})}\sin(\theta_k-\theta_i)\ne0.
\]
This is a symbolic certificate for every tree size, not an extrapolation from sampled trees.
\end{remark}

\subsection{From tree observations to ordinary persistence}
Differentiating a pairwise dot product under \eqref{eq:flow} gives
\begin{equation}\label{eq:qdot}
 \dot q_{ij}=\dot p_i\cdot p_j+p_i\cdot\dot p_j
 =(1-q_{ij})(s_i+s_j).
\end{equation}
Away from collisions, $(q_{ij},s_i+s_j)\mapsto(q_{ij},\dot q_{ij})$ is an invertible smooth observation change. If $d_e$ is the half-geodesic length of edge $e$, then
\begin{equation}\label{eq:deathchange}
 q_e=\cos(2d_e),\qquad \dot q_e=-2\sin(2d_e)\dot d_e.
\end{equation}
Away from coincident and antipodal endpoints, this is likewise an invertible smooth change between the edge first jets. Within a strict chamber, sorting the MST deaths is a fixed permutation.

\begin{corollary}[First-jet identifiability]\label{cor:jet}
At every generic reference formation satisfying the regularity conditions of Section~\ref{sec:setting}, the map
\[
 J=(H,L_fH)
\]
has quotient differential of rank $2n-3$. It is locally injective on the quotient, and has a locally Lipschitz inverse on its image.
\end{corollary}
\begin{proof}
Equations \eqref{eq:qdot}--\eqref{eq:deathchange} transfer Theorem~\ref{thm:rank} to $J$. In a local quotient chart, select $2n-3$ observation coordinates whose Jacobian minor is nonsingular. The ordinary inverse function theorem applied to these coordinates gives the result.
\end{proof}

\begin{remark}[What becomes observable]
This first jet means the sorted finite deaths and their time derivatives under the \emph{known} flow. The derivatives are additional observations, not attributes of one diagram. One snapshot takes values in $\R^{n-1}$, whereas the quotient has dimension $2n-3>n-1$ for $n\ge3$; it cannot provide local injectivity on an open quotient neighborhood. The $n-1$ edge rows are independent at the witness of Theorem~\ref{thm:rank}; the same genericity argument gives generic rank $n-1$ for the tree edge-length differential, leaving static fibers of dimension $n-2$. Temporal information removes these generic infinitesimal ambiguities for this particular flow.
\end{remark}

\begin{corollary}[Two sufficiently close snapshots]\label{cor:two}
At each reference formation of Corollary~\ref{cor:jet}, there exist a common quotient neighborhood $U$ and $h_0>0$ such that
\[
 A_h(u)=\big(H(u),H(\varphi_h(u))\big)
\]
is injective on $U$ with locally Lipschitz inverse for every $0<|h|<h_0$.
\end{corollary}
\begin{proof}
Let $B=L_fH$ and work in a local chart. Since the flow and the observation are smooth within the chamber,
\begin{equation}\label{eq:Rh}
 R_h(u)=\left(H(u),\frac{H(\varphi_h(u))-H(u)}{h}\right)
 \longrightarrow (H(u),B(u))=J(u)
\end{equation}
in $C^1$ on a sufficiently small compact chart neighborhood. For $h\ne0$, $R_h$ is an invertible linear change of $A_h$. The next theorem supplies a uniform neighborhood and an injective lower Lipschitz bound.
\end{proof}

\section{Sampling-rate conditioning}\label{sec:conditioning}
Two elementary local metric facts make the preceding statement quantitative. At a formation with trivial rotation--permutation stabilizer, a smooth quotient chart satisfies, after shrinking,
\begin{equation}\label{eq:metricchart}
 \ell\|u-v\|_2\le\df(u,v)\le L\|u-v\|_2
\end{equation}
for constants $\ell,L>0$. One justification is to compare the maximum product distance with the $\ell^2$ product of spherical geodesic distances: their quotient distances differ by at most a factor $\sqrt n$. The latter is locally the distance of the smooth Riemannian quotient by the free compact rotation action; standard chart norm equivalence applies. The finite permutation action introduces no new local identifications at a trivial-stabilizer point.

Second, for sufficiently close diagrams in the present chamber,
\begin{equation}\label{eq:matching}
 \DB(D_0(u),D_0(v))=\|H(u)-H(v)\|_\infty.
\end{equation}
To see this, choose the neighborhood so all finite deaths are bounded below by $a>0$ and the cost of sorted matching is less than $a/2$. Sending any finite point $(0,d)$ to the diagonal costs $d/2\ge a/2$, so an optimal matching uses no diagonal. Sorted matching of real death coordinates minimizes the bottleneck cost. The essential classes match at zero cost. The same argument holds uniformly at the second observation time after shrinking $U$ and $h_0$.

\begin{theorem}[Sharp order of inverse conditioning]\label{thm:rate}
Fix a generic reference $u_0$, set $m=n-1$, and write
\[
 \sigma=\sigma_{\min}(DJ(u_0))>0,
\]
using Euclidean chart and observation norms. There are a common convex chart ball $U\ni u_0$ and $h_0>0$ such that, for all $u,v\in U$ and $0<|h|<h_0$,
\begin{equation}\label{eq:upper}
 \df(u,v)\le
 \frac{2L\sqrt m}{\sigma}\sqrt{1+\frac4{h^2}}\,\Delta_h(u,v),
\end{equation}
where
\[
 \Delta_h(u,v)=\max\{\DB(D_0(u),D_0(v)),
 \DB(D_0(\varphi_hu),D_0(\varphi_hv))\}.
\]
Furthermore, the smallest constant $K_h(U)$ for which
$\df(u,v)\le K_h(U)\Delta_h(u,v)$ holds on $U$ satisfies
\[
 K_h(U)=\Theta(1/|h|)\qquad(h\to0,\ h\ne0).
\]
Section~\ref{sec:certificate} supplies numerical certificates at specified references.
\end{theorem}

\begin{proof}
By \eqref{eq:Rh}, continuity, and shrinking a common convex ball and time interval, arrange
\[
 \sup_{u\in U}\|DR_h(u)-DJ(u_0)\|_{2\to2}\le\sigma/2.
\]
Integrate the derivative along the segment from $u$ to $v$, keeping the fixed matrix $DJ(u_0)$ separate. The triangle inequality gives
\[
 \|R_h(u)-R_h(v)\|_2\ge(\sigma/2)\|u-v\|_2.
\]
This step uses closeness to one fixed injective matrix; a pointwise singular-value lower bound alone would not justify the integrated inequality. By \eqref{eq:matching}, the first block of the difference in $R_h$ has Euclidean norm at most $\sqrt m\Delta_h$, and the second at most $2\sqrt m\Delta_h/|h|$. Combining these estimates with \eqref{eq:metricchart} proves \eqref{eq:upper}, hence $K_h(U)=O(1/|h|)$.

For the lower bound, set $B=L_fH$. Since $d=2n-3>m$, choose a unit vector $w\ne0$ in $\ker DH(u_0)$. Full rank of $DJ(u_0)$ forces $DB(u_0)w\ne0$; set $\kappa=\|DB(u_0)w\|_\infty>0$. Smoothness yields
\[
 D(H\circ\varphi_h)(u_0)w=h\,DB(u_0)w+O(h^2),
\]
with remainder bounded by $Ch^2$ uniformly for small $h$. For $v_t=u_0+tw$, first let $t\to0$ with $h$ fixed. Equations \eqref{eq:metricchart}--\eqref{eq:matching} imply that any inverse constant on $U$ obeys
\[
 K_h(U)\ge\frac{\ell}{|h|\kappa+Ch^2}\ge\frac{c}{|h|}
\]
for some $c>0$ and all sufficiently small nonzero $h$. This proves the matching lower order.
\end{proof}

\section{Validated neighborhoods and sampling intervals}\label{sec:certificate}
\subsection{Pinned coordinates and a verified left inverse}
Vector norms and induced operator norms in this section are Euclidean,
except where an infinity norm or an $\infty\!\to\!2$ norm is specified.
Put the reference in the gauge $p_0=e_3$, $p_1=(x,0,z)$ with $x>0$.
The $3n\times d$ matrix $Z$ has a zero block at vertex $0$, the unit tangent
$(z,0,-x)$ at vertex $1$, and an orthonormal tangent pair at each remaining
vertex. Columns at different vertices have disjoint supports, and each tangent frame
is orthonormal; hence $Z^{\mathsf T}Z=I$.
The normalized chart is
\begin{equation}\label{loc:eq:chart}
 \chi_i(u)=\frac{p_i+(Zu)_i}{\sqrt{1+\|(Zu)_i\|^2}},\qquad
 \|u\|_2<\rho\le\min(1/2,x/2).
\end{equation}
In particular $\|D\chi(u)\|\le1$ and
$\df([\chi(u)],[\chi(v)])\le\|u-v\|_2$.

Let $T$ be the reference unique MST, with edges ordered by increasing
length. For its edges define the polynomial maps
\[
 Q_e(p)=p_i\cdot p_j,\qquad
 B_e(p)=(1-Q_e(p))(s_i+s_j).
\]
Direct differentiation gives $L_fQ=B$. Set
\[
 J_q(u)=(Q(\chi(u)),B(\chi(u))),\qquad A=DJ_q(0),
\]
and require $A$ to have full column rank. Theorem~\ref{thm:rank}
gives this generically; the implementation checks the actual reference.
Define the exact left inverse
\begin{equation}\label{loc:eq:leftinverse}
 L=(A^{\mathsf T}A)^{-1}A^{\mathsf T}=[L_0\ L_1],\qquad LA=I_d,
\end{equation}
where each block has $m$ columns. For the supplied rational references,
every entry in \eqref{loc:eq:leftinverse} is rational. Any exactly verified left
inverse could be used in the theorem.

\subsection{A reference-specific validation theorem}
For a candidate geometric trajectory radius $D>0$, let
\[
 X_D=\prod_{i,k}[p_{ik}-D,p_{ik}+D]\cap[-1,1]^{3n}.
\]
The interval box contains every sphere configuration whose vertices are
within geodesic distance $D$ of their reference vertices. Let $V,\lambda,H_2$
be verified bounds, on $X_D\cap(S^2)^n$, satisfying
\begin{equation}\label{loc:eq:localbounds}
 \max_i\|f_i\|\le V,\qquad
 \|Df\|\le\lambda,\qquad
 \|L_1D F_2\|\le H_2,\qquad F_2=L_fB.
\end{equation}
Derivatives here are ambient derivatives of the polynomial extensions.
For implementation, an explicit expression is
\begin{equation}\label{loc:eq:acceleration}
 (F_2)_e=(1-Q_e)\bigl(\dot s_i+\dot s_j-(s_i+s_j)^2\bigr),
 \qquad
 \dot s_i=f_i\cdot S+p_i\cdot\sum_k f_k.
\end{equation}
This identity holds for the ambient polynomial vector field as well.

Let $\alpha$ be a verified bound
\begin{equation}\label{loc:eq:spatial}
 \sup_{\|u\|_2\le\rho}\|L DJ_q(u)-I_d\|\le\alpha.
\end{equation}
The implementation encloses this derivative on the larger coordinate cube
$[-\rho,\rho]^d$. The trajectory bound is applied only to the Euclidean
coordinate ball, whose initial vertex displacements are at most $\rho$.

\begin{theorem}[Local interval recovery criterion]\label{loc:thm:local}
Suppose the reference rank and chart conditions hold, the bounds
\eqref{loc:eq:localbounds}--\eqref{loc:eq:spatial} are validated, and $\tau>0$ obeys
\begin{equation}\label{loc:eq:conditions}
 \rho+\tau V<D,\qquad \lambda\tau<1,\qquad
 \theta:=\alpha+\frac{\tau H_2}{2(1-\lambda\tau)}<1.
\end{equation}
Suppose also that the persistence-branch and matching conditions in
Section~\ref{loc:sec:chamber} hold on the trajectory neighborhood.
Then for every $0<|h|\le\tau$, the two ordinary $H_0$ diagrams at times
$0,h$ are injective on
$U=\{[\chi(u)]:\|u\|_2<\rho\}$.
If $c_j$ is any certified upper bound for
$\|L_j\|_{\infty\to2}$, then for $[q],[r]\in U$,
\begin{equation}\label{loc:eq:inverse}
 \df([q],[r])\le K_h\Delta_h(q,r),\qquad
 K_h=\frac{2}{1-\theta}\left(c_0+\frac{2c_1}{|h|}\right),
\end{equation}
where $\Delta_h$ is the larger of the two corresponding bottleneck distances.
The set $U$ contains the full formation-metric ball about $[p]$ of radius
\begin{equation}\label{loc:eq:eta}
 \eta=\frac{\rho x}{12\lceil\sqrt n\rceil}.
\end{equation}
\end{theorem}

\begin{proof}
First, every trajectory starting in the coordinate ball remains within
geodesic distance $D$ of the reference vertices for $|t|\le\tau$.
Indeed, before the first exit each trajectory lies in $X_D$, so its vertex
speeds are at most $V$. The initial displacement is at most $\rho$, and
the displacement at a first exit would be at most $\rho+\tau V<D$, a
contradiction. The same argument applies in backward time.

Define the normalized two-time dot-product observation
\[
 R_h^q(u)=\left(Q(\chi(u)),
          \frac{Q(\varphi_h\chi(u))-Q(\chi(u))}{h}\right).
\]
The second block is the time average of $B$ along the trajectory.
Let $Y_t=D\varphi_t(\chi(u))D\chi(u)$. The variational equation gives
$\|Y_t\|\le e^{\lambda|t|}$ on the validated neighborhood. Moreover,
\[
 \frac{d}{dt}\bigl(L_1 DB(\varphi_t\chi(u))Y_t\bigr)
 =L_1D F_2(\varphi_t\chi(u))Y_t.
\]
Integrating this identity and then averaging yields
\[
 \|L(DR_h^q(u)-DJ_q(u))\|
 \le\frac{|h|H_2}{2}e^{\lambda|h|}
 \le\frac{\tau H_2}{2(1-\lambda\tau)}.
\]
The last step uses $e^a\le(1-a)^{-1}$ for $0\le a<1$.
Combining with \eqref{loc:eq:spatial} gives
\[
 \sup_{\|u\|\le\rho}\|D(LR_h^q)(u)-I_d\|\le\theta<1.
\]
Integrate along a straight segment in the coordinate ball, keeping the
identity matrix separate from the error. Then
\begin{equation}\label{loc:eq:coercivity}
 \|LR_h^q(u)-LR_h^q(v)\|_2\ge(1-\theta)\|u-v\|_2.
\end{equation}
This proves injectivity on the entire ball, rather than just a rank test
at its center.

The matching conditions imply that bottleneck distance is exactly the
maximum discrepancy in corresponding ordered death coordinates. Since
$Q_e=\cos(2\,\mathrm{death}_e)$, its derivative with respect to the death
coordinate has absolute value at most $2$. If $\eps=\Delta_h(q,r)$,
each dot-product difference vector at one time has infinity norm at most
$2\eps$. Consequently
\[
 \|LR_h^q(u)-LR_h^q(v)\|_2
 \le 2c_0\eps+4c_1\eps/|h|.
\]
Combine this with \eqref{loc:eq:coercivity} and the chart Lipschitz bound to
obtain \eqref{loc:eq:inverse}.

For completeness, the explicit metric-ball coverage follows by alignment.
If $\df([p],[q])<\eps<\eta$, rotate and relabel $q$ so all vertex
displacements are below $\eps$. Rotate its first vertex to $e_3$, moving
every vertex by at most another $\eps$. The second vertex's planar
projection is then within $2\eps$ of $(x,0)$, so its azimuth has absolute
value at most $\arcsin(2\eps/x)\le4\eps/x$.
An axial correction puts it in the positive reference meridian and
leaves all displacements below $6\eps/x\le1/2$.
The inverse gnomonic coordinates have lengths at most twice these
displacements. Their combined norm is below
$12\sqrt{n-1}\eps/x<\rho$, proving \eqref{loc:eq:eta}.
The hypotheses $\rho\le1/2,x/2$ ensure every small-angle inequality used.
\end{proof}

\subsection{Preserving the persistence branches}\label{loc:sec:chamber}
Only the MST and its ordered edge lengths need to remain fixed. Ordering
changes between other pairs do not alter the observed $H_0$ branches.
For each reference pair $e$, let $q_e=p_i\cdot p_j$, and choose a
validated upper bound $s_e^+\ge\sqrt{1-q_e^2}$. Define
\begin{equation}\label{loc:eq:qrange}
 I_e(D)=[q_e-w_e,q_e+w_e]\cap[-1,1],\qquad
 w_e=2s_e^+D+2D^2.
\end{equation}
If each vertex moves by geodesic distance at most $D$, its pair length
changes by at most $2D$. Taylor's theorem for cosine, with second
derivative bounded by one, proves \eqref{loc:eq:qrange}.
Write $\underline q_e,\overline q_e$ for its endpoints. The following
strict checks suffice:
\begin{enumerate}
 \item $-1<\underline q_e\le\overline q_e<1$ for every pair;
 \item for consecutive ordered tree edges $e_k,e_{k+1}$,
       $\underline q_{e_k}>\overline q_{e_{k+1}}$;
 \item for each non-tree pair $e$, with tree path $P_T(e)$,
       $\overline q_e<\min_{a\in P_T(e)}\underline q_a$;
 \item $5D<a_*$, where $a_*=(1-\max_e q_e)/2$.
\end{enumerate}
The third condition is the strict MST cycle condition, expressed in
decreasing dot-product order. It preserves the unique MST. The second
preserves the sorted death branches. The first prevents collisions and
antipodal pairs. Unique MST edge lengths also exclude nontrivial
permutation isometries for $n\ge3$: such an isometry must fix every tree
edge by its distinct length, and hence every vertex. Together with the
pinned meridian this gives a regular local quotient chart.

To justify the fourth condition, every reference finite death is at least
$a_*$ because $\arccos q\ge1-q$. Every current death is therefore at
least $a_*-D$, while matching corresponding sorted deaths in any two
configurations costs at most $2D$. Since
$2D<(a_*-D)/2$, an optimal bottleneck matching cannot send a finite point
to the diagonal. Sorting minimizes the bottleneck cost among remaining
matchings. The essential classes match with cost zero.

\section{Computational methodology and exact results}\label{sec:computation}
\subsection{Study design and reference formations}
The computational study has three objectives: instantiate the local
recovery criterion at explicit formations, compare its certified domains
with the analytic baseline, and check finite instances of the structural
proof. The primary outputs are the formation radius $\eta$, admissible
delay $\tau$, derivative bound $\theta$, and inverse bound $K_\tau$.
The calculations use exact rational arithmetic and outward interval
bounds. They do not integrate sample trajectories, fit an empirical
conditioning rate, or test a reconstruction algorithm on noisy data.
The flow is \eqref{eq:flow} with its stated unit coupling and reference
time $0$. The delay $\tau$ is a uniform guarantee for every
$0<|h|\le\tau$, rather than a numerical integration step size.

Four regular references use prefixes of length $n=3,4,6,8$ of the
following deterministic stereographic parameter list:
\begin{gather*}
 (a_i,b_i)=
 (0,0),(1/2,0),(1/5,3/5),(-3/5,2/7),\\
 (-2/5,-4/5),(4/5,-3/5),(5/4,2/5),(-3/2,-1/4),\\
 P(a,b)=\frac{(2a,2b,1-a^2-b^2)}{1+a^2+b^2},\qquad p_i=P(a_i,b_i).
\end{gather*}
Indices start at zero. These are illustrative configurations of quotient
dimensions $3,5,9,13$, not a random sample of formations. All have
$p_0=e_3$ and $p_1=(4/5,0,3/5)$, so $x=4/5$ in the pinned chart.
Writing $D_i=1+a_i^2+b_i^2$, rational orthonormal tangent frames are
\begin{align*}
 t_{i,1}&=(1-2a_i^2/D_i,-2a_ib_i/D_i,-2a_i/D_i),\\
 t_{i,2}&=(-2a_ib_i/D_i,1-2b_i^2/D_i,-2b_i/D_i).
\end{align*}
The matrix $Z$ has no free component at vertex $0$, uses the single
pinned tangent $(3/5,0,-4/5)$ at vertex $1$, and uses this pair at every
remaining vertex. Unit lengths, tangency, and $Z^{\mathsf T}Z=I$ are
checked exactly before constructing \eqref{loc:eq:chart}.

For a centered control, let $S_3$ be the sum of the first three points
and append their reflections across the plane perpendicular to $S_3$:
\[
 p_{i+3}=p_i-\frac{2(p_i\cdot S_3)}{\|S_3\|^2}S_3,
 \qquad i=0,1,2.
\]
This gives six rational, noncoincident, nonantipodal points with sum
exactly zero. For a nearly centered positive reference, retain the first
three points and use the following parameters for the appended points:
\begin{align*}
 (a_3,b_3)&=(-1349/1172,-1065/1172)+(4\varepsilon,5\varepsilon),\\
 (a_4,b_4)&=(36/145,-438/145)+(5\varepsilon,6\varepsilon),\\
 (a_5,b_5)&=(-213/74,165/74)+(6\varepsilon,7\varepsilon),
 \qquad \varepsilon=1/10000.
\end{align*}
The unperturbed pairs are the stereographic coordinates of the reflected
points. These formulas completely specify both six-point formations;
the accompanying records additionally store their Cartesian coordinates.
The centered control has exact first-jet rank $8<9$ for the ordered tree
edges $(0,1),(3,4),(1,2),(4,5),(2,5)$ and has repeated pair distances.
It therefore lies outside both the full-rank
and strict-chamber assumptions, and is not an isolated test of centering
within the hypotheses of the recovery theorem.

\subsection{Exact reference calculations}
For each of the five positive references, the program forms every
rational pair dot product and checks that all are distinct and belong
to $(-1,1)$. Kruskal's algorithm processes edges in decreasing dot-product
order, which is exactly increasing spherical-distance order, to construct
the unique MST. Its ordered edges index the finite $H_0$ deaths
$\tfrac12\arccos(p_i\cdot p_j)$ under the filtration convention of
Section~\ref{sec:setting}. Thus neither numerical evaluation of
$\arccos$ nor a general-purpose persistence package is needed for the
certificate decisions.

The reference Jacobian $A=DJ_q(0)$ is computed by rational
differentiation of $Q_e$ and $B_e=(1-Q_e)(s_i+s_j)$ in the pinned chart.
A separate product-rule calculation checks the temporal rows using the
specified vector field. Rational elimination computes
$G^{-1}=(A^{\mathsf T}A)^{-1}$; the identities $GG^{-1}=I$ and
$LA=I$, with $L=G^{-1}A^{\mathsf T}$, are then checked exactly.
Full column rank is established by these identities, without a
floating-point singular-value threshold. The centered control is handled
separately by exact rank elimination.

For the analytic baseline in Appendix~\ref{app:baseline}, the code
constructs
\[
 \underline\sigma=2^{-32}
 \left\lfloor\frac{2^{32}}{\sqrt{\operatorname{tr}(G^{-1})}}\right\rfloor
 \le \sigma_{\min}(A).
\]
The integer is obtained with exact integer arithmetic, and
$\underline\sigma>0$ and
$\underline\sigma^2\operatorname{tr}(G^{-1})\le1$ are checked.
Together with the minimum pair-order, collision, and antipodal margins,
this determines the baseline pair $(\rho_0,\tau_0)$ and the corresponding
formation radius $\eta_0$. The interval refinement below starts from
these same baseline values in every positive case.

\subsection{Interval enclosures and candidate acceptance}
The interval kernel represents endpoints in $2^{-96}\mathbb Z$.
Rational inputs are rounded outward; addition is exact on this grid,
while multiplication and division use integer floor and ceiling, and
square roots use outward-rounded integer square roots. Display decimals
and elapsed times have no role in acceptance.
First-order interval automatic differentiation encloses $DJ_q$ on
$[-\rho,\rho]^d$ and $Df,DF_2$ on the ambient box $X_D$ from
Section~\ref{sec:certificate}. The normalized chart uses the exact identity
$1+\|(Zu)_i\|^2=1+\sum_{k\in\mathcal I_i}u_k^2$ in its denominator.
The cube provides a derivative enclosure for the Euclidean coordinate
ball; the trajectory guarantee applies to that ball.

For entrywise absolute upper bounds $P_{ij}$ on a matrix enclosure $E$,
the implemented operator-norm bound is
\[
 \|E\|_2\le\min\left\{
 \sqrt{\sum_{ij}P_{ij}^2},\,
 \sqrt{(\max_i\sum_jP_{ij})(\max_j\sum_iP_{ij})}\right\}.
\]
This bounds $LDJ_q-I$, $Df$, and $L_1DF_2$, with $F_2$ evaluated by
\eqref{loc:eq:acceleration}. Sphere-specific bounds $V\le n$ and
$\lambda\le3n$ tighten looser interval estimates. Outward enclosures of
\[
 \sqrt{\sum_a\left(\sum_b|(L_j)_{ab}|\right)^2}
\]
provide the constants $c_j$ in \eqref{loc:eq:inverse}.

For each candidate $(\rho,\tau)$, the trajectory-tube search starts at
$D=2(\rho+n\tau)$. Four passes compute a speed bound $V$ on the current
$X_D$ and replace $D$ by the smaller of its current value and
$\tfrac{101}{100}(\rho+\tau V)$. The code then recomputes $V$ at the
final $D$ and requires the strict first-exit inequality
$\rho+\tau V<D$. The iteration only selects a candidate tube; this
final inequality, together with the first-exit argument in
Theorem~\ref{loc:thm:local}, supplies the trajectory guarantee.

Acceptance also requires $0<\rho\le\min(1/2,x/2)$, $\tau>0$,
$\lambda\tau<1$, and the four strict branch and matching conditions
of Section~\ref{loc:sec:chamber}: pair dots remain in $(-1,1)$,
consecutive MST edges remain ordered, every non-tree edge satisfies the
MST cycle inequality, and $5D<a_*$. These checks preserve the required
MST branch, without requiring all non-tree edges to retain their mutual
order. Finally the search imposes
\[
 \theta=\alpha+\frac{\tau H_2}{2(1-\lambda\tau)}\le\frac12,
\]
a stronger sufficient condition than $\theta<1$ in the theorem.
The accepted pair determines $\eta$ and $K_\tau$ through
\eqref{loc:eq:eta} and \eqref{loc:eq:inverse}.

\subsection{Deterministic search and reported results}
The same two-stage search is used for all five positive references.
First, it tests $(s\rho_0,s\tau_0)$ at $s=1,2,4,\ldots$ until the
first rejection, then performs nine bisection tests between the last
accepted multiplier and the rejected multiplier. The initial lower
endpoint is zero. If the final lower endpoint is $s_-$, set
$s=(19/20)s_-$, $\rho=s\rho_0$, and
$\tau_{\mathrm{base}}=s\tau_0$.
Second, with this $\rho$ fixed, test delays
$t\tau_{\mathrm{base}}$ at $t=2,4,8,\ldots$, with initial lower endpoint
$1$, until rejection, and again perform nine bisection tests.
If the resulting lower endpoint is $t_-$, choose
$t=\max\{1,(19/20)t_-\}$ and
$\tau=t\tau_{\mathrm{base}}$.
The first and second expansion loops allow at most 18 and 24 tests,
respectively; failure to encounter a rejection within the cap raises
an error. The initial lower endpoints are search conventions, not
additional acceptance tests.

All bounds and strict inequalities are recomputed for the
selected final pair, whose acceptance is asserted before it is recorded.
The $95\%$ factors are search margins, not substitutes for this check.
The recorded runs required 155 candidate evaluations in total, including
five final recomputations. There is no random sampling in configuration
selection or certificate search. Bisection here is a candidate-selection
heuristic: it supplies no optimality claim about the largest radius,
delay, or inverse constant. Rejection means that these sufficient tests
do not certify the candidate.

\begin{table}[htbp]
\centering\small
\setlength{\tabcolsep}{5pt}
\begin{tabular}{lrrrr}
\toprule
Reference & $\eta_{\rm print}$ & $\tau_{\rm print}$ & $\theta$ upper & $K_{\tau_{\rm print}}$ upper \\
\midrule
$n=3$ & $1.872\times10^{-4}$ & $2.104\times10^{-3}$ & $0.496$ & $1.04\times10^{4}$ \\
$n=4$ & $5.773\times10^{-5}$ & $5.519\times10^{-4}$ & $0.497$ & $1.98\times10^{5}$ \\
$n=6$ & $3.027\times10^{-5}$ & $4.322\times10^{-4}$ & $0.498$ & $1.73\times10^{5}$ \\
$n=8$ & $5.993\times10^{-6}$ & $1.204\times10^{-4}$ & $0.499$ & $2.66\times10^{6}$ \\
Near-centered $n=6$ & $6.671\times10^{-9}$ & $3.443\times10^{-4}$ & $0.499$ & $1.90\times10^{9}$ \\
\bottomrule
\end{tabular}
\caption{Validated local recovery for every formation in the ball of radius
$\eta_{\rm print}$ and every $0<|h|\le\tau_{\rm print}$. Radii (radians) and
time endpoints (consensus units) are rounded downward; derivative errors and
inverse bounds are rounded upward. The last column is evaluated at the
printed endpoint. For smaller delays use formula~\eqref{loc:eq:inverse},
not the endpoint constant. Exact records retain the larger unrounded
neighborhoods and delays.}
\label{tab:certificates}
\end{table}

The exact reference ranks are $3,5,9,13$ in the four regular cases and
$9$ in the nearly centered case. The interval refinement increases the
certified formation radii in the regular cases by factors of about
$44$--$135$ and the admissible delays by about $60$--$426$, relative to
Appendix~\ref{app:baseline}. These gains occur simultaneously at the
selected radius--delay pairs. The nearly centered case has a delay gain
of about $5.78\times10^5$, reflecting a much smaller local speed bound,
but its formation radius remains below $7\times10^{-9}$ radians.
A larger admissible time certificate therefore does not imply good
spatial robustness. None of the displayed inverse bounds is asserted
to be an optimal condition number.

The table is generated from the rational certificate records.
Formation radii and delays are rounded downward to four significant
figures, and $\theta$ upward to three decimal places. Each printed
inverse bound is evaluated at the printed, downward-rounded delay and
then rounded upward to three significant figures. For a smaller
observation gap, the applicable bound is $K_h$ in
\eqref{loc:eq:inverse}, rather than the displayed endpoint value.

\subsection{Independent checks and their scope}
The verification program replays all five final candidates and compares
every recomputed rational scalar bound with its saved value. For each
reference, the interval chart Jacobian at $\rho=0$ must contain the exact
rational matrix $A$; the acceleration derivative $DF_2$ must contain the
result of an independently written Hessian/product-rule calculation.
Multiplying both accepted $\rho$ and $\tau$ by $100$ gives a deliberately
oversized candidate that must fail the sufficient conditions. These
negative checks test rejection behavior, not nonidentifiability.
The interval kernel is additionally checked on 300 seeded rational
cases, and the automatic-differentiation kernel on the exact value and
derivative of $(1+u^2)^{-1/2}$ at $u=3/4$.

Separately, the witness program tests a path, star, binary tree, and one
seeded random Pr\"ufer tree at each $n=3,4,8,16,32$, giving 20 cases.
It constructs the rational pole--equator witness used in the structural
proof and differentiates $(q_e,s_i+s_j)$; its equivalence to the temporal
jet is the row transformation in that proof. Each case checks the
factorized minor's block identities, nonzero sine factors, and the
oriented and unsigned incidence determinants. For $n\le8$, the complete
dense rational determinant is also computed and compared with the
factorization (12 cases). These witnesses need not realize their
specified trees as MSTs. The finite checks support the implementation
and examples; Theorem~\ref{thm:rank} establishes the result for all trees
and all sizes. Appendix~\ref{app:software} specifies the software entry
points, seeds, records, and reproduction commands.

\section{Discussion, obstructions, and scope}\label{sec:scope}
\begin{proposition}[Centered nongeneric obstruction]\label{prop:centered}
If $S(p)=0$, then for every tree $T$,
\[
 \rank\big(DF_T(p)|_{T_p\mathcal C_n}\big)\le n+2.
\]
Consequently, at a centered configuration with three-dimensional rotation orbit, full quotient rank is impossible for $n\ge6$.
\end{proposition}
\begin{proof}
The $n-1$ edge-dot rows contribute at most $n-1$ to the rank. At $S=0$, $\delta s_i=p_i\cdot\delta S$, so every remaining row is $(p_i+p_j)\cdot\delta S$. These rows factor through the three-dimensional vector $\delta S$, and contribute at most three more. For $n\ge6$, $n+2<2n-3$. Such configurations are stationary under \eqref{eq:flow}; the rank argument also concerns perturbations away from the centered set.
\end{proof}

This is a first-order obstruction; the rank bound alone does not exclude
nonlinear local injectivity.

\paragraph{Reflection ambiguity.}
For any reflection $Q\in O(3)\setminus\SO$, equivariance gives
$\varphi_t(Qp)=Q\varphi_t(p)$. Hence the entire persistence time series agrees for $p$ and $Qp$. If the unordered formation is not congruent to its reflection under $\SO$, these are distinct formation classes. The theorems assert local recovery on a branch, not global recovery modulo $\SO$. Local recovery also descends to the $O(3)$ quotient; global injectivity there is not established.

The results concern a fixed number of points, an exactly specified vector field,
and known sampling times. They do not establish global reconstruction,
robustness to model mismatch, or recovery under unknown coupling gains.
The inverse bounds compare two admissible model trajectories in the stated
neighborhood; they do not by themselves construct an estimator for arbitrary
noisy diagrams. In particular, a small barcode discrepancy alone does not
certify that the unknown formation lies in the local branch.

The structural theorem identifies the missing information precisely: one
Lie derivative completes the infinitesimal constraints supplied by every
tree. The conditioning theorem quantifies the cost of approximating this
derivative with two observations. The interval theorem makes that same
argument verifiable on an entire neighborhood. Extensions to other
interaction graphs or unknown dynamics require a new rank analysis;
the present proof does not transfer automatically.

\paragraph{Reproducibility.}
Exact reference coordinates, rational Jacobians and left inverses, interval
endpoints, and verification programs accompany this manuscript. The
certificate programs require only the Python standard library. The code and
data are supplied as a supplement; a public archival identifier has not been
assigned. Proofs are given here, and software checks are additional internal
verification, not proof-assistant formalization or external peer review.

\appendix
\section{A closed-form baseline certificate}\label{app:baseline}
The local interval theorem can be compared with an entirely analytic
sufficient bound. Use the pinned chart of Section~\ref{sec:certificate},
write $J_q=(Q\circ\chi,B\circ\chi)$, and let
$A=DJ_q(0)$ have a certified lower singular-value bound
$\underline\sigma>0$. Assume every reference pair dot product is distinct
and strictly between $-1$ and $1$. With minima over all unordered pairs,
put
\begin{align*}
 \mu&=\tfrac12\min\{1-\max_e q_e,\ 1+\min_e q_e,
                         \min_{e\ne f}|q_e-q_f|\},\\
 M&=\lceil\sqrt{n-1}\rceil,\qquad N=\lceil\sqrt n\rceil,\\
 \rho_0&=\min\{1/2,x/2,\mu/16,\underline\sigma/(248nM)\},\\
 \tau_0&=\min\{1/(3n),\mu/(16n),
                    \underline\sigma/[6n^2M(13N+24)]\}.
\end{align*}
For all $0<|h|\le\tau_0$ and $\|u\|,\|v\|<\rho_0$, one has
\begin{equation}\label{eq:baseline-inverse}
 \df([\chi(u)],[\chi(v)])\le
 \frac{4\sqrt{n-1}}{\underline\sigma}\sqrt{1+4/h^2}\,
 \Delta_h(\chi(u),\chi(v)).
\end{equation}
The chart image contains the full formation ball of radius
$\eta_0=\rho_0x/(12N)$.

Here are the estimates needed for a self-contained derivation. All
operator norms use Euclidean product norms. At sphere configurations,
\[
 \|f\|\le n\sqrt n,\quad \|Df\|\le3n,\quad
 \|DQ\|\le\sqrt{2m},\quad \|D^2Q\|\le\sqrt m,
 \quad m=n-1.
\]
For the derivative of $f$, writing $V=\sum_i v_i$ gives
\[
 Df_i[v]=(I-p_ip_i^{\mathsf T})V-s_iv_i-p_i(v_i\cdot S).
\]
The stacked first term is bounded by $n\|v\|$ and the remaining block
operator by $2n\|v\|$. For a single edge put
$a=(p_i+p_j)\cdot S$ and $q=p_i\cdot p_j$. Then
\[
 |a|\le2n,\quad \|Dq\|\le\sqrt2,\quad \|D^2q\|\le1,
 \quad \|Da\|\le\sqrt2 n+2\sqrt n,\quad
 \|D^2a\|\le2\sqrt{2n}.
\]
The product rule for $B_e=(1-q)a$ yields
\[
 \|DB_e\|\le4\sqrt2 n+4\sqrt n\le8n,
 \qquad \|D^2B_e\|\le6n+8\sqrt{2n}\le13n,
\]
where $n\ge3$. Stacking edges gives $\|DB\|\le8n\sqrt m$ and
$\|D^2B\|\le13n\sqrt m$.
For $\mathcal N(y)=y/\|y\|$,
\[
 D^2\mathcal N(y)[a,b]=
 -\frac{a(y\cdot b)+b(y\cdot a)+y(a\cdot b)}{\|y\|^3}
 +\frac{3y(y\cdot a)(y\cdot b)}{\|y\|^5}.
\]
Thus $\|D\chi\|\le1$, $\|D^2\chi\|\le6$, and the chain rule gives
\[
 \|D^2(Q\circ\chi)\|<10\sqrt m,\qquad
 \|D^2(B\circ\chi)\|\le61n\sqrt m,
 \qquad\|D^2J_q\|\le62nM.
\]

For the normalized two-time dot-product map $R_h^q$ in the proof of
Theorem~\ref{loc:thm:local}, the variational equation gives
$\|D\varphi_tD\chi\|\le e^{3n|t|}$. Differentiating $DB$ along the
trajectory, using the product rule and the preceding estimates, and
then averaging over the time interval yields
\[
 \|DR_h^q(u)-DJ_q(u)\|
 \le\frac{|h|}{2}n^2M(13N+24)e^{3n|h|}.
\]
The choices of $\rho_0,\tau_0$ make the spatial and temporal errors
relative to $A$ at most $\underline\sigma/4$ each, since
$e^{3n|h|}\le e<3$. Integration on the coordinate segment therefore
gives
\[
 \|R_h^q(u)-R_h^q(v)\|_2
       \ge\tfrac12\underline\sigma\|u-v\|_2.
\]

To transfer this inequality to diagrams, every pair half-length moves
from its reference value by at most
$\delta=\rho_0+n\tau_0\le\mu/8$. Each of the collision, antipodal,
and pair-order margins in half-length coordinates is at least $\mu$,
because $q=\cos(2d)$ is $2$-Lipschitz. Hence every pair order is
preserved. If $a$ is the minimum reference half-length, all current
deaths are at least $7a/8$, while sorted matching of any two current
diagrams costs at most $a/4$. A diagonal match costs at least
$7a/16>a/4$ and is excluded. Each single-time dot-product discrepancy
is then at most $2\sqrt m\Delta_h$ in Euclidean norm. This proves
\eqref{eq:baseline-inverse}. The alignment argument in
Theorem~\ref{loc:thm:local} proves the stated metric-ball coverage.

For rational reference points and rational orthonormal tangent frames,
$A$ and $G=A^{\mathsf T}A$ are rational. Exact inversion verifies rank,
and
\[
 \sigma_{\min}(A)\ge[\operatorname{tr}(G^{-1})]^{-1/2}.
\]
A positive rational $\underline\sigma$ is certified by the exact test
$\underline\sigma^2\operatorname{tr}(G^{-1})\le1$. The implementation
uses a 32-bit dyadic lower square-root bound, computed by integer
arithmetic. For a reference where it vanishes, the precision must be
increased before a positive certificate is claimed. These global
estimates explain the baseline used to initialize the interval search;
they are sufficient bounds rather than optimized conditioning estimates.

\section{Software and reproduction protocol}\label{app:software}
All code used in this study can be found here: \url{https://github.com/DemetersSon83/Quotient-Geometry-Persistence-Stable-Metrics}

\bibliographystyle{unsrtnat}
\bibliography{bibliography}
\end{document}